\documentclass[twoside,11pt]{article}
\usepackage{mhStyle2}

\usepackage[cmex10]{amsmath}
\usepackage{amssymb,url,xspace,caption}
\usepackage{booktabs}
\usepackage[ruled,vlined]{algorithm2e}
\usepackage{algorithmic}
\usepackage{array,balance}
\usepackage{times}
\usepackage{multirow}
\usepackage{enumitem}
\usepackage{pdfsync}
\usepackage[usenames,dvipsnames]{xcolor}

\usepackage[pagebackref=false,breaklinks=true,letterpaper=true,colorlinks,bookmarks=false]{hyperref}

\graphicspath{{./}{./figures/}}
\DeclareGraphicsExtensions{.pdf,.jpeg,.png,.jpg}

\makeatletter
\DeclareRobustCommand\onedot{\futurelet\@let@token\@onedot}
\def\@onedot{\ifx\@let@token.\else.\null\fi\xspace}

\def\eg{\emph{e.g}\onedot} 
\def\ie{\emph{i.e}\onedot}

\def\etal{\emph{et al}\onedot}
\makeatother

\def\Vec#1{{\boldsymbol{#1}}}
\def\Mat#1{{\boldsymbol{#1}}}

\def\GRASS#1#2{\mathcal{G}({#1},{#2})}

\def\SPD#1{\mathcal{S}_{++}^{#1}}

\newcommand{\tr}{\mathop{\rm  Tr}\nolimits}

\def\cite{\Citep}

\jmlrheading{}{2015}{Harandi \etal}
\ShortHeadings{Beyond Gauss: Image-Set Matching on the Riemannian Manifold of PDFs}{Harandi \etal}

\begin{document}

\title{Beyond Gauss: Image-Set Matching on the Riemannian \\Manifold of PDFs}

\author{\name Mehrtash Harandi \email mehrtash.harandi@nicta.com.au \\
       \addr NICTA and Australian National University\\
       Canberra, Australia
       \AND
       \name Mathieu Salzmann \email mathieu.salzmann@nicta.com.au \\
   	   \addr NICTA and Australian National University\\
       Canberra, Australia
       \AND
       \name Mahsa Baktashmotlagh \email m.baktashmotlagh@qut.edu.au \\
   	   \addr Queensland University of Technology\\
       Brisbane, Australia       }

\maketitle

\begin{abstract}

State-of-the-art image-set matching techniques typically implicitly model each image-set with a Gaussian distribution.
Here, we propose to go beyond these representations and model image-sets as probability distribution functions (PDFs) using kernel density estimators. 
To compare and match image-sets, we exploit Csisz\'{a}r $f$-divergences, which bear strong connections to the
geodesic distance defined on the space of PDFs, \ie, the statistical manifold. Furthermore, 
we introduce valid positive definite kernels on the statistical manifolds, which let us make use of more powerful classification schemes to match image-sets. 
Finally, we introduce a supervised dimensionality reduction technique that learns a latent space where $f$-divergences reflect the class labels of the data. Our experiments on diverse problems, such as video-based face recognition and dynamic texture classification, evidence the benefits of our approach over the state-of-the-art  image-set matching methods.

\end{abstract}

\section{Statistical Manifolds and $f$-Divergences}
\label{sec:background}

In this paper, we rely on probability density functions (PDFs) and on the distances between them to analyze image-sets. In this section, we therefore review some concepts related to the geometry of the space of PDFs.

Let $\mathcal{X}$ be a set. A  (PDF) on $\mathcal{X}$ is a function $p:\mathcal{X} \to \mathbb{R}^+$ 
such that $\int_\mathcal{X} p(x)dx = 1$. Let $\mathcal{M}$ be a family of PDFs on the set $\mathcal{X}$.
{
With certain assumptions (\eg, differentiability), $\mathcal{M}$ forms a Riemannian structure, \ie, 
a differentiable manifold equipped with a Riemannian metric. 
The Riemannian metric enables us to measure the length of curves\footnote{%
On a Riemannian manifold, the geodesic distance between two points corresponds to the length of the shortest path on the manifold between them.
}.
}

The Fisher-Rao metric~\cite{Rao_1945} is undoubtedly the 
most common Riemannian metric to analyze $\mathcal{M}$.
Unfortunately, an analytic form of geodesic distance induced by the Fisher-Rao metric can only be obtained for specific
distributions, such as Gaussians, or the exponential family~\cite{Rao_1945}\footnote{Not to be confused with exponential distributions, or distributions generally expressed as sums of exponentials.}.
In other words, estimating geodesic distances between general distributions, such as the ones we use here, is impractical. To address this issue, here, we propose to compare PDFs with two $f$-divergences, which, as discussed later, have strong connections with the geodesic distance.

Formally, a Csisz\'{a}r $f$-divergence is a function of two probability distributions that measures their similarity, and is defined as
\begin{equation}
\delta_f(p \Vert q) = \int f\left(\frac{p}{q}\right) dq\;,
\label{eqn:f_divergence}
\end{equation}
where $f$ is a convex function on $(0,\infty)$ with $f(1) = 0$. Different choices of $f$ yield different divergences. Below, and in the rest of this paper, we focus on two special cases, which induce the Hellinger distance and the Jeffrey divergence, respectively.

The Hellinger distance can be obtained by choosing $f(t) = (\sqrt{t} -1)^2$ in Eq.~\ref{eqn:f_divergence}, and is defined below.
\begin{definition}
The Hellinger distance between two probability distributions $p$ and $q$ is defined as 
\begin{equation}
	\delta_{H}^2(p \Vert q) = \int \Big(\sqrt{p(\Vec{x})} -\sqrt{q(\Vec{x})}\Big)^2d\Vec{x}\;.
	\label{eqn:Hellinger_dist}
\end{equation}
\end{definition}

If, instead, we set $f(t) = t\ln(t) - \ln(t)$ in Eq.~\ref{eqn:f_divergence}, we obtain the Jeffrey divergence defined below.
\begin{definition}
The Jeffrey or symmetric KL divergence between two probability distributions $p$ and $q$ is defined as 
\begin{equation}
	\delta_{J}(p \Vert q) = 
	\int \Big(p(\Vec{x}) - q(\Vec{x})\Big)\ln{\frac{p(\Vec{x})}{q(\Vec{x})}}d\Vec{x}\;.
	\label{eqn:symm_KL_dist}
\end{equation}
\end{definition}

From a geometrical point of view, the Riemannian structure induced by the Hellinger distance is different from the one 
induced by the J-divergence. However, these two divergences share the property that their respective Riemannian metrics can be obtained from the Fisher-Rao metric (see Thereom~5 in~\cite{Amari_2010}). Furthermore, in~\cite{Baktashmotlagh_CVPR_2014}, it was shown that the length of any given curve is the same under the Hellinger distance and under the Fisher-Rao metric up to scale. These two properties therefore relate these divergences to the geodesic distance on the statistical manifold, and thus make them an attractive alternative to compare PDFs.
Another important property of these two divergences is given by the following theorem.
{
\begin{theorem}
The Hellinger distance and the Jeffrey divergence between two distributions are invariant under differentiable and invertible (diffeomorphism) transformations.
\label{thm:invariance_f_div}
\end{theorem}
That is, given two distributions $p_1(\Vec{x})$ and $p_2(\Vec{x})$ in space $\mathcal{X}; \Vec{x} \in \mathcal{X}$, 
let $h : \mathcal{X} \to \mathcal{Y}$ be a differentiable and invertible function that converts $\Vec{x}$ into
$\Vec{y}$. Under function $h$, we have $p_i(\Vec{x})d\Vec{x} = q_i(\Vec{y})d\Vec{y}; i \in \{1,2\}$ and
$d(\Vec{y}) = |\mathcal{J}(\Vec{x})| d\Vec{x}$ where 
$|\mathcal{J}(\Vec{x})|$ denotes the determinant of the Jacobian matrix $h$. The above invariance property states that
\begin{equation*}
\delta_f(p_1(\Vec{x}),p_2(\Vec{x})) = \delta_f(q_1(\Vec{y}),q_2(\Vec{y})).
\end{equation*}
}
The proof of this theorem can be found in several recent studies (\eg, see Theorem.1 in~\cite{Qiao_TSP_2010}). It has also been known to mathematicians for decades~\cite{Ali_1966}.
Invariance to diffeomorphism seems an attractive property in the context of computer vision, and in particular for image-set matching, since images in a set can typically be subject to many variations,
such as changes of illumination or of
environment/capture conditions.
Furthermore, we will exploit this property when deriving our dimensionality reduction method in Section~\ref{sec:dim_red}. Note that the affine invariance of some metrics on the SPD manifold, which has made such metrics  popular, is a lesser form of this property. In other words, the $f$-divergences are invariant to a broader set of transformations, including affine ones. 

Before concluding this section, it is worth mentioning 
the following relationship.

\begin{remark}
The Bhattacharya distance between two distributions is defined as $\delta_B(p \Vert q) = -\ln(B(p,q))$ where 
$B(\cdot,\cdot)$ is the Bhattacharya coefficient defined as 
\begin{equation*}
B(p,q) = \int \sqrt{p(\Vec{x})q(\Vec{x})}d\Vec{x}\;.
\end{equation*}
The Bhattacharya coefficient and the Hellinger distance are related through $\delta_H^2(p\Vert q) = 2 - 2B(p,q)$. Interestingly, the Bhattacharya distance between two zero mean multivariate normal distribution $p = \mathcal{N}(\Vec{0},\Mat{\Sigma}_p)$
and $q = \mathcal{N}(\Vec{0},\Mat{\Sigma}_q)$ is proportional to the Stein divergence between their corresponding covariance 
matrices. More specifically, $\delta_B(p,q) = 0.5 S(\Mat{\Sigma}_p,\Mat{\Sigma}_q)$ which can be proved by direct 
insertion. The Stein divergence has been successfully used in addressing several problems in vision~\cite{Cherian_PAMI_2013,Zhang_AAAI_2015}.
\end{remark}

\section{Image-Sets as PDFs}
\label{sec:practical_considerations}

We now introduce our approach to modeling image-sets as PDFs. To this end, let $\{\Vec{x}_i\}_{i=1}^{n}$ be a set of $n$ images, where each $\Vec{x}_i \in \mathbb{R}^D$ is a feature vector describing one image in the set. We propose to make use of Kernel Density Estimation (KDE) to obtain a fine-grained estimate $\hat{p}(\Vec{x})$ of the distribution $p(\Vec{x})$ of the features. This estimate can be written as 
\begin{footnotesize}
\begin{equation}
\hat{p}(\Vec{x}) = \frac{1}{n\sqrt{\det(2\pi \Sigma)}}\sum_{i = 1}^{n} 
\exp \bigg(-\frac{1}{2} \Big(\Vec{x} - \Vec{x}_i \Big)^T \Sigma^{-1} \Big(\Vec{x} - \Vec{x}_i \Big) \bigg)\;.
\label{eqn:kde_px}
\end{equation}
\end{footnotesize}

Given two image-sets $\{\Vec{x}_i^{(p)}\}_{i=1}^{n_p}$ and $\{\Vec{x}_i^{(q)}\}_{i=1}^{n_q}$, Eq.~\ref{eqn:kde_px} provides us with the means to estimate their respective PDFs $p(\Vec{x})$ and $q(\Vec{x})$. We then aim to compare these image-sets by computing the statistical distances between $\hat{p}(\Vec{x})$ and $\hat{q}(\Vec{x})$. As discussed in Section~\ref{sec:background}, we propose to rely on $f$-divergences to compare $p(\Vec{x})$ and $q(\Vec{x})$. Note, however, that the integrals corresponding to the Hellinger distance and to the J-divergence do not have an analytic solution for our KDE representation. Therefore, below, we derive solutions to compute a robust estimate of these two divergences.

\subsection{Empirical $f$-Divergences}

Let us first consider the case of the Hellinger distance. Given two sets of samples $\{\Vec{x}_i^{(p)}\}_{i=1}^{n_p}$ and $\{\Vec{x}_i^{(q)}\}_{i=1}^{n_q}$ drawn from $p(\Vec{x})$ and $q(\Vec{x})$, respectively, directly estimating the integral of Eq.~\ref{eqn:Hellinger_dist} is not straightforward. To make this easier, one can rewrite Eq.~{\ref{eqn:Hellinger_dist} as
\begin{align}
\delta_H^2(p \Vert q) &= \int \bigg(1 - \sqrt{\frac{q(\Vec{x})}{p(\Vec{x})}} \bigg)^2 p(\Vec{x}) d\Vec{x} \\
&= E_{p} \bigg(1 - \sqrt{\frac{q(\Vec{x})}{p(\Vec{x})}} \bigg)^2\;,
\label{eqn:est_hellinger_px}
\end{align}
where $E_{p}(\cdot)$ denotes the expectation under $p$. Following the strong law of large numbers, as is commonly done in practice, such an expectation can then be estimated as
\begin{align}
\hat{\delta}_H^2(p \Vert q) &= \frac{1}{n_p}\sum_i^{n_p} 
\bigg(1 - \sqrt{\frac{\hat{q}(\Vec{x}_i^{(p)}))}{\hat{p}(\Vec{x}_i^{(p)})}} \bigg)^2 \;,
\label{eq:est_hell_px}
\end{align}
with $\hat{p}(\cdot)$ and $\hat{q}(\cdot)$ obtained by KDE. 
Unfortunately, such an estimate would in general be different if one had chosen to make use of $q(\Vec{x})$ instead of $p(\Vec{x})$ to derive the expectation of Eq.~\ref{eqn:est_hellinger_px}. This implies that the resulting estimate of the Hellinger distance would be asymmetric, and thus poorly-suited to our goals.
It is now clear that 
if the realizations of $q(\Vec{x})$, \ie $\{\Vec{x}_l^{(q)}\}_{l=1}^{n_q}$ are used to estimate $\delta_H^2(p \Vert q)$,
a different result is attained as in this case we have 
$\hat{\delta}_H^2(p \Vert q) = E_{q} \bigg(1 - \sqrt{\frac{p(\Vec{x})}{q(\Vec{x})}} \bigg)^2$ and hence

\begin{align}
\hat{\delta}_H^2(p \Vert q) &= \frac{1}{n_q}\sum_i^{n_q} 
\bigg(1 - \sqrt{\frac{\hat{p}(\Vec{x}_i^{(q)})}{\hat{q}(\Vec{x}_i^{(q)})}} \bigg)^2 \;.
\label{eqn:est_hellinger_qx}
\end{align}

While one could compute the average of Eqn.~\eqref{eqn:est_hellinger_px} and Eqn.~\eqref{eqn:est_hellinger_qx} to obtain 
a less biased and symmetric estimation of $\delta_H^2(p \Vert q)$, we follow the approach of~\cite{Carter_2009} to obtain a more robust estimate of the Hellinger distance. More specifically, we rewrite $\delta_H^2(p \Vert q)$ as 
\begin{footnotesize}
\begin{align}
&\delta_H^2(p \Vert q) = \hspace{-1ex}\int 
\hspace{-1ex}\bigg(\sqrt{\frac{p(\Vec{x})}{p(\Vec{x}) + q(\Vec{x})}} - \sqrt{\frac{q(\Vec{x})}{p(\Vec{x}) + q(\Vec{x})}} \bigg)^2 
\hspace{-1ex}(p(\Vec{x}) + q(\Vec{x})) d\Vec{x} \notag \\
&= E_{p} \Big(\sqrt{T(\Vec{x})} - \sqrt{1 - T(\Vec{x})}\Big)^2 + 
E_{q} \Big(\sqrt{T(\Vec{x})} - \sqrt{1 - T(\Vec{x})}\Big)^2\;,
\label{eqn:hellinger_cont}
\end{align}
\end{footnotesize}
with 
\begin{equation}
T(\Vec{x}) = \frac{p(\Vec{x})}{p(\Vec{x}) + q(\Vec{x})}\;.
\label{eqn:t_x}
\end{equation}

Given our two sets of samples $\{\Vec{x}_i^{(p)}\}_{i=1}^{n_p}$ and $\{\Vec{x}_i^{(q)}\}_{i=1}^{n_q}$, this allows us to estimate the Hellinger distance as
\begin{align}
&\hat{\delta}_H^2(p \Vert q) = \frac{1}{n_p}\sum_i^{n_p} 
\bigg(\sqrt{T(\Vec{x}_i^{(p)})}  - \sqrt{1 - T(\Vec{x}_i^{(p)})} \bigg)^2 \notag \\ 
&+\frac{1}{n_q}\sum_i^{n_q}
\bigg(\sqrt{T(\Vec{x}_i^{(q)})}  - \sqrt{1 - T(\Vec{x}_i^{(q)})} \bigg)^2\;.
\label{eqn:hellinger_discrete}
\end{align}

The benefits of this approach are twofold. First, the resulting estimate is symmetric. Second, and maybe more importantly, the denominator of 
 $T(\cdot)$ alleviates the potential instabilities that low probabilities of samples under either $\hat{q}$ or $\hat{p}$ would have resulted in by making use of the formulation in Eq.~\ref{eq:est_hell_px}, or of its counterpart in terms of $q$. Note that such low probabilities are quite common when relying on KDE with high-dimensional data.
 


In the case of the Jeffrey divergence, we make use of the same idea as for the Hellinger distance. We therefore express the J-divergence in terms of $T(\cdot)$, which, after some derivations, yields
\begin{align}
\hat{\delta}_J(p \Vert q) &= \frac{1}{n_p}\sum_i^{n_p}
(2T(\Vec{x}_i^{(p)})-1)\ln\frac{T(\Vec{x}_i^{(p)})}{1-T(\Vec{x}_i^{(p)})} \notag \\
&+ \frac{1}{n_q}\sum_i^{n_q}
(2T(\Vec{x}_i^{(q)})-1)\ln\frac{T(\Vec{x}_i^{(q)})}{1-T(\Vec{x}_i^{(q)})}
\;.
\end{align}

Our two empirical estimates give us practical ways to evaluate the distance between two image-sets represented by their PDFs. Given a training set of $m$ image-sets and a query image-set, matching can then simply be achieved by computing the distance of the query to all training image-sets, and choosing the nearest one as matching set.


\section{Kernels on the Statistical Manifold}
\label{sec:kernels}

In the previous section, we have introduced an approach to comparing the distributions of two image-sets using empirical estimates of the Hellinger distance or of the J-divergence. Such an approach, however, only allows us to make use of a nearest neighbor classifier. Kernel methods (\eg, Kernel Fisher discriminant analysis), however, provide much more powerful tools to perform classification, and thus image-set matching. Here, we therefore turn to the question of whether the divergences defined in Section~\ref{sec:background} can generate valid positive definite (\emph{pd}) kernels.
To answer this question, let us first define the notion of pd kernels.
\begin{definition}[Real-valued Positive Definite Kernels]
Let $\mathcal{X}$ be a nonempty set. A symmetric function $k: \mathcal{X} \times \mathcal{X} \to \mathbb{R}$ is a positive definite (\emph{\textbf{pd}})
kernel on $\mathcal{X}$ if and only if
\begin{equation}
	\sum_{i,j=1}^nc_ic_jk(x_i,x_j) \geq 0
	\label{eqn:pd_kernel_def}
\end{equation}
for any $n \in \mathbb{N}$, $x_i \in \mathcal{X}$ and $c_i \in \mathbb{R}$. 
\end{definition}

For the J-divergence, the kernel 
\begin{equation}
k_J(p,q) = \exp(-\sigma \delta_J(p \Vert q))
\label{eqn:Jeffrey_Gaussian_kernel}
\end{equation}
was introduced in~\cite{Moreno_NIPS_2003}, although
without a formal proof of positive definiteness. 
Later, in~\cite{Hein_2005}, a conditionally positive definite (\emph{cpd}) kernel based on a smooth version of the J-divergence was derived. 
To the best of our knowledge, a counter-example that disproves the  positive definiteness of $k_J(\cdot,\cdot)$ has never been exhibited in the literature. Therefore, in our experiments, we assumed that $k_J(\cdot,\cdot)$ is \emph{pd}.
{
Please note that the kernel in Eq.\eqref{eqn:Jeffrey_Gaussian_kernel} is not a Laplace kernel since 
$\delta_J$ denotes the  J-divergence not $\delta_J^2$. The reason behind our choice
of $\delta_J$ over $\delta_J^2$ is that the unlike the Hellinger distance, J-divergence is not a metric 
(this can be shown by a counter-example).
}

In the case of the Hellinger distance, a conditionally positive definite (\emph{cpd}) kernel was studied in~\cite{Hein_2005}. Here, in contrast, we derive valid \emph{pd} kernels. More precisely, we do not only introduce a single \emph{pd} kernel, but rather provide a recipe to generate diverse \emph{pd} kernels on the statistical manifold. Our derivations rely on the definition of negative definite kernels given below.
\begin{definition}[Real-valued Negative Definite Kernels]
Let $\mathcal{X}$ be a nonempty set. A symmetric function $\psi: \mathcal{X} \times \mathcal{X} \to \mathbb{R}$ is a negative definite (\emph{\textbf{nd}})
kernel on $\mathcal{X}$ if and only if $\sum_{i,j=1}^nc_ic_jk(x_i,x_j) \leq 0$ for any $n \in \mathbb{N}$, $x_i \in \mathcal{X}$ and $c_i \in \mathbb{R}$ with $\sum_{i=1}^n c_i = 0$.
\end{definition}
Note that, in contrast to positive definite kernels, an additional constraint of the form $\sum c_i = 0$ is required in the negative definite case. Given this definition, we now prove that the Hellinger distance is \emph{nd}. 
\begin{theorem}[Negative Definiteness of the Hellinger distance] \label{thm:hellinger_nd}
Let $\mathcal{M}$ denote the statistical manifold. The Hellinger distance $\delta_H^2:\mathcal{M} \times \mathcal{M} \to \mathbb{R}^+$ is negative definite.
\end{theorem}
\begin{proof}
\begin{align*}
&\sum_{i,j=1}^N c_i c_j \delta_H^2(p_i \Vert p_j) = \sum_{i,j=1}^N c_i c_j \int_x \Big(\sqrt{p_i(x)} - \sqrt{p_j(x)}\Big)^2 dx \\
&= \sum_{i=1}^N c_i \sum_{j=1}^N c_j \int_x p_j(x) dx + \sum_{j=1}^N c_j \sum_{i=1}^N c_i \int_x p_i(x) dx \\
&-2\sum_{i,j=1}^N c_i c_j \int_x \sqrt{p_i(x)p_j(x)} dx \\
&= -2\int_x  \sum_{i}^N c_i \sqrt{p_i(x)} \sum_{j}^N c_j \sqrt{p_j(x)}  dx \\
&= -2\int_x  \|\sum_{i}^N c_i \sqrt{p_i(x)}\|^2  dx \leq 0\;,
\end{align*}%
where the terms in the second line have disappeared due to the constraints $\sum_i c_i = 0$, resp. $\sum_j c_j = 0$, and to the fact that the integrals have value 1 for any $i$, resp. $j$.
\end{proof}

We then make use of the following theorem, which originated from the work of I. J. Schoenberg (1903-1990).
\begin{theorem} [Theorem 2.3 in Chapter 3 of~\cite{Berg:1984}] 
\label{thm:laplace_kernel}
Let $\mu$ be a probability measure on the half line $\mathbb{R}^{+}$ and  
$0 < \int_{0}^{\infty}t\mathrm{d}\mu(t) < \infty$. Let $\mathcal{L}_{\mu}$ be the Laplace transform of $\mu$, \ie, 
$\mathcal{L}_\mu(s) = \int_{0}^{\infty}e^{-ts}\mathrm{d}\mu(t),\; s \in \mathbb{C}$. Then, $\mathcal{L}_\mu(\beta f)$
is positive definite for all $\beta > 0$  if and only if $f:\mathcal{X} \times \mathcal{X} \to \mathbb{R}^{+}$  is negative definite.
\end{theorem}
Theorem~\ref{thm:laplace_kernel} provides a general recipe to create \emph{pd} kernels. In particular, here, we focus on the Gaussian and the Laplace kernels, which have proven powerful in Euclidean space.
The Gaussian kernel can be obtained by choosing $\mu(t) = \delta(t-1)$ in Theorem~\ref{thm:laplace_kernel}, where $\delta$ denotes the Dirac function.  On the statistical manifold, this kernel can then be written as
\begin{equation}
k_H(p,q) = \exp(-\sigma \delta_H^2(p,q)),~~~\sigma > 0\;.
\label{eqn:Hellinger_Gaussian_kernel}
\end{equation}

To derive the Laplace kernel on the statistical manifold, we must further rely on the following theorem.
\begin{theorem}[Corollary 2.10 in Chapter 3 of~\cite{Berg:1984}] \label{thm:power_nd}
If $\psi:\mathcal{X} \times \mathcal{X} \to \mathbb{R}$ is negative definite and satisfies $\psi(\Vec{x},\Vec{x}) \geqq 0$
then so is $\psi^\alpha$ for $0 < \alpha < 1$.
\end{theorem}
As a consequence, by choosing $\psi = \delta_H^2$ and $\alpha = 1/2$ in Theorem~\ref{thm:power_nd}, we have that $\delta_H(\cdot,\cdot)$ is \emph{nd}. Then, applying Theorem~\ref{thm:laplace_kernel} with $\delta_H(\cdot,\cdot)$ and $\mu(t) = \delta(t-1)$ lets us derive the Laplace kernel on the statistical manifold
\begin{align}	
	k_L(p,q) = \exp(-\sigma \delta_H(p,q)),~~~\sigma > 0\;.
	\label{eqn:Hellinger_Laplace_kernel}
\end{align}

\begin{remark}
The Hellinger distance can be thought of as the chordal distance between points on an infinite-dimensional unit hyper-sphere. More specifically, 
the square root function is a diffeomorphism between the statistical manifold and the unit hyper-sphere. In~\cite{Srivastava_CVPR_2007}, this was exploited to estimate the distance between discretized PDFs as the geodesic distance on the corresponding (finite-dimensional) hyper-sphere. Such a distance, however, cannot induce a valid positive definite Gaussian kernel, since the Gaussian kernel produced by the geodesic distance on a Riemannian manifold is not positive definite unless the manifold is flat~\cite{Feragen_2014}. In contrast, as shown above, our divergences yield valid positive definite kernels, which allow us to exploit more sophisticated classification methods.
\end{remark}

\begin{remark}
Note that the discussion above proves the positive definiteness of kernels defined with the exact Hellinger distance, and does not necessarily extend to its empirical estimate. However, since the strong law of large numbers guarantees convergence of our empirical estimate to the true distance, given sufficiently many samples, the resulting empirical kernels will also be \emph{pd}.
\end{remark}

In our experiments, we used the kernels derived above to perform kernel discriminant analysis. Image-set matching was then achieved by using the Euclidean distance in the resulting low-dimensional latent space.

\section{$f$-Divergences for Dimensionality Reduction}
\label{sec:dim_red}

The methods described in Sections~\ref{sec:practical_considerations} and~\ref{sec:kernels} directly compare the distributions of the original features of each image-set. As mentioned earlier, with high-dimensional features that are common in computer vision, KDE may produce very sparse PDFs (\ie, PDFs that are strongly peaked around the samples and zero everywhere else), which may be less reliable to compare. To address this issue, here, we propose to learn a mapping of the features to a low-dimensional space, such that the $f$-divergences in the resulting space reflect some interesting properties of the data. 

{
As will be shown shortly, we will formulate the dimensionality reduction as a problem on the Grassmann manifolds. 
The use of Grassmann and Stiefel 
manifolds for dimensionality reduction is an emerging topic in machine learning. Two notable examples are robust PCA using Grassmannian 
by Hauberg \etal~\cite{Hauberg_CVPR_2014} and linear dimensionality reduction using Stiefel manifolds 
by Cunningham and Ghahramani~\cite{Cunningham_JMLR_2015}.}

Focusing on the supervised scenario, we search for a latent space where two image-sets are close to each other (according to the $f$-divergence) if they belong to the same class and far apart if they don't.
That is, given a set of image-sets $\mathcal{X} = \left\{\Mat{X}_1,\cdots, \Mat{X}_m \right\}$, where each image-set 
$\Mat{X}_i = \{\Vec{x}^{(i)}_{l}\}_{l=1}^{n_i},~ \Vec{x}^{(i)}_l \in \mathbb{R}^D$, 
our goal is to find a transformation $\Mat{W} \in \mathbb{R}^{D \times d}$ such that the $f$-divergences between the mapped image-sets $\left\{\{\Mat{W}^T\Vec{x}^{(i)}_l\}_{l=1}^{n_i}\right\}_{i=1}^m$ encode some interesting structure of the data. Here, we represent this structure via an affinity function $a(\Mat{X}_i,\Mat{X}_j)$ that encodes pairwise relationships between the image-sets. This affinity function will be described in Section~\ref{sec:affinity}.

Since we aim for the $f$-divergences to reflect this affinity measure, we can write a cost function of the form
\begin{equation}
	\mathcal{L}(\Mat{W})  =  \sum_{i,j} a({\Mat{X}_i,\Mat{X}_j}) \cdot \delta \left(\Mat{W}^T\Mat{X}_i,\Mat{W}^T\Mat{X}_j\right)\;,
	\label{eqn:DR_cost_fun}
\end{equation}
where $\delta: \mathcal{M} \times \mathcal{M} \to \mathbb{R}^+$ is either $\delta_H^2(\cdot,\cdot)$ or $\delta_{J}(\cdot,\cdot)$, and where we sum over all pairs of image-sets.
To avoid possible degeneracies when minimizing this cost function w.r.t. $\Mat{W}$, and following common practice in dimensionality reduction, we enforce the solution to be orthogonal, \ie, $\Mat{W}^T\Mat{W} = \mathbf{I}_{d}$. This allows us to write dimensionality reduction as the optimization problem
\begin{eqnarray}
\Mat{W}^* = &\underset{\Mat{W}}{\operatorname{argmin}}&\hspace{-0.2cm} \mathcal{L}(\Mat{W}) \label{eqn:DR_opt_prob} \\
& {\rm s.t.} &\hspace{-0.5cm}\Mat{W}^T \Mat{W} = \Mat{I}_d\;.\notag
    \end{eqnarray}
Below, we show that, for both divergences of interest, \ie, the Hellinger distance and the J-divergence, \eqref{eqn:DR_opt_prob} is a minimization problem on a Grassmann manifold.

The Grassmann manifold $\GRASS{d}{D}$ is the space of $d$-dimensional subspaces in $\mathbb{R}^D$ and corresponds to
as a quotient space of the Stiefel manifold (\ie, the space of  $d$-dimensional frames in $\mathbb{R}^D$, or in other words orthogonal $D\times d$ matrices)~\cite{Edelman:1998}.
More specifically, the points on the Stiefel manifold that form a basis of the same subspace are identified with a single point on the Grassmann manifold.
As such, a minimization problem with orthogonality constraint $\Mat{W}^T\Mat{W} = \mathbf{I}_{d}$ is a problem on the Grassmannian
\emph{iff} the cost of the problem is invariant to the choice of basis of the subspace spanned by $\Mat{W}$. Mathematically, 
$\min_\Mat{W} \mathcal{L}(\Mat{W})$ with $\Mat{W}^T\Mat{W} = \mathbf{I}_{d}$ is a problem on the Grassmannian \emph{iff} 
$\mathcal{L}(\Mat{W}\Mat{R})= \mathcal{L}(\Mat{W}),~\forall \Mat{R} \in \mathcal{O}(d)$, where $\mathcal{O}(d)$ denotes the group of $d \times d$ orthogonal matrices. 
Since transformations in $\mathcal{O}_d$ are bijections, the invariance property of Theorem~\ref{thm:invariance_f_div} directly shows that the cost function of~\eqref{eqn:DR_opt_prob} is invariant to the choice of basis. In other words,~\eqref{eqn:DR_opt_prob} can be solved as an unconstrained minimization problem on $\GRASS{d}{D}$.

In practice, to solve~\eqref{eqn:DR_opt_prob} on $\GRASS{d}{D}$, we make use of Newton-type methods (\eg, the conjugate gradient method). These methods inherently require the gradient of $\mathcal{L}(\Mat{W})$.
On $\GRASS{d}{D}$, the gradient\footnote{%
On an abstract Riemannian manifold $\mathcal{M}$, the gradient of a smooth real function $f$ at a point
$x \in \mathcal{M}$, denoted by $\nabla f(x)$, is the element of $T_x\mathcal{M}$ satisfying
$\langle \nabla f(x), \zeta \rangle_x = Df(x)[\zeta],~\forall \zeta \in T_x\mathcal{M}$, where $Df(x)[\zeta]$ denotes
the directional derivative of $f$ at $x$ along direction $\zeta$.}
can be expressed as
\begin{equation}
\nabla_\Mat{W}\mathcal{L}(\Mat{W}) = (\mathbf{I}_D - \Mat{W}\Mat{W}^T)\operatorname{grad}\mathcal{L}(\Mat{W}),
\end{equation}
where $\operatorname{grad}\mathcal{L}(\Mat{W})$ is the $D \times d$ matrix of partial derivatives of $\mathcal{L}(\Mat{W})$ 
with respect to the elements of $\Mat{W}$, \ie,
\begin{equation*}
[\operatorname{grad}\mathcal{L}(\Mat{W})]_{i,j} = \frac{\partial\mathcal{L}(\Mat{W})}{\partial\Mat{W}_{i,j}}.
\end{equation*}
The detailed derivations of $\operatorname{grad}\mathcal{L}(\Mat{W})$ for our two $f$-divergences are provided in the appendix. 

\begin{remark}
A weaker form of constraints on $\Mat{W}$ could be achieved by exploiting the geometry of the oblique manifold~\cite{Absil_2006}. The oblique manifold $\mathcal{OB}(d,D)$ is a submanifold of $\mathbb{R}^{D \times d}$ 
whose points are rank $d$ matrices $\Mat{W}$ subject to
\begin{equation*}
	\mathrm{ddiag}\big(\Mat{W}^T\Mat{W} \big) = \mathbf{I}_d\;,
\end{equation*}
where $\mathrm{ddiag}(\Mat{A})$ denotes the diagonal matrix whose elements are those of the diagonal of $\Mat{A}$.
Similarly to the Stiefel and Grassmann manifolds, a point on the oblique manifold is represented by a tall $D \times d$ matrix with unit norm
columns. By contrast, however, these columns are not necessarily orthogonal. We found empirically that relying on these weaker constraints does not yield as good accuracy for image-set matching as our orthogonality constraints. This appears to be due to the fact that, when two columns of $\Mat{W}$ become very similar, but still satisfy the rank constraint, the cost decreases, but the resulting latent space is less informative. This illustrates the importance of orthogonal columns, which favor diversity in the latent space.
\end{remark}

\subsection{Affinity Measure}
\label{sec:affinity}

As mentioned above, we propose to exploit supervised data to define the affinity measure used in the cost function of Eq.~\ref{eqn:DR_cost_fun}. Note that unsupervised approaches are also possible, for instance to find a mapping that preserves the closeness of pairs of image-sets. 

In our supervised scenario, let $y_i$ denote the class label of image-set $\Mat{X}_i$, with $1 \leq y_i \leq C$. We define the affinity between two sets $\Mat{X}_i$ and $\Mat{X}_j$ with labels $y_i$ and $y_j$, respectively, as
\begin{equation}
	a(\Mat{X}_i,\Mat{X}_j) = g_w(\Mat{X}_i,\Mat{X}_j) - g_b(\Mat{X}_i,\Mat{X}_j) \;,
	\label{eqn:affinity_Gw_Gb_combined}
\end{equation}
where $g_w$ and $g_b$ encode a notion of within-class similarity and between-class similarity, respectively. These similarities can be expressed as
\begin{equation}
   	g_w(\Mat{X}_i,\Mat{X}_j) =
   	\left\{
 		\begin{matrix}
       	1, & \mbox{if} \; \Mat{X}_i \in N_w(\Mat{X}_j ) \; \mbox{~or~} \; \Mat{X}_j \in N_w(\Mat{X}_i ) \\
   		0, & \mbox{otherwise}
   		\end{matrix}
   	\right.
\nonumber 
\end{equation}%
\noindent
\begin{equation}
   	g_b(\Mat{X}_i,\Mat{X}_j) =
   	\left\{
   	\begin{matrix}
       	1, & \mbox{if} \; \Mat{X}_i \in N_b(\Mat{X}_j ) \; \mbox{~or~} \; \Mat{X}_j \in N_b(\Mat{X}_i ) \\
       	0, & \mbox{otherwise}
   	\end{matrix}
   	\right.
   	\nonumber 
\end{equation}%
\noindent
where $N_w(\Mat{X}_i)$ is the set of $\nu_w$ nearest neighbors of $\Mat{X}_i$ (according to the $f$-divergence) that share the same label as $y_i$, 
and $N_b(\Mat{X}_i)$ contains the $\nu_b$ nearest neighbors of $\Mat{X}_i$ having different labels. 
In our experiments, we defined $\nu_w$ as the minimum number of points in a class and found $\nu_b \leq \nu_w$ by cross-validation.

{
Before moving to the next section, we would like to take a detour and discuss a couple of observations/insights on the 
dimensionality reduction algorithm. Examples below are all taken from the face recognition experiment on YouTube Celebrity (YTC) 
dataset~\cite{YT_Celebrity} (the very first experiment in \textsection~\ref{sec:experiments}).
Fig.~\ref{fig:cost_function_example} shows the behavior of the cost function depicted in Eq.\eqref{eqn:DR_cost_fun} for a few
iterations of the conjugate gradient method on Grassmannian. In practice, 
we found that the algorithm would converge in most cases in less than 25 iterations.
On a related note, each iteration of the conjugate gradient on Grassmannian for the aforementioned case took near 40 seconds 
on a quad core desktop machine. 
}

\def \CFSIZE {0.7}
\begin{figure}[!tb]
  \centering
  \includegraphics[width=\CFSIZE \columnwidth,keepaspectratio]{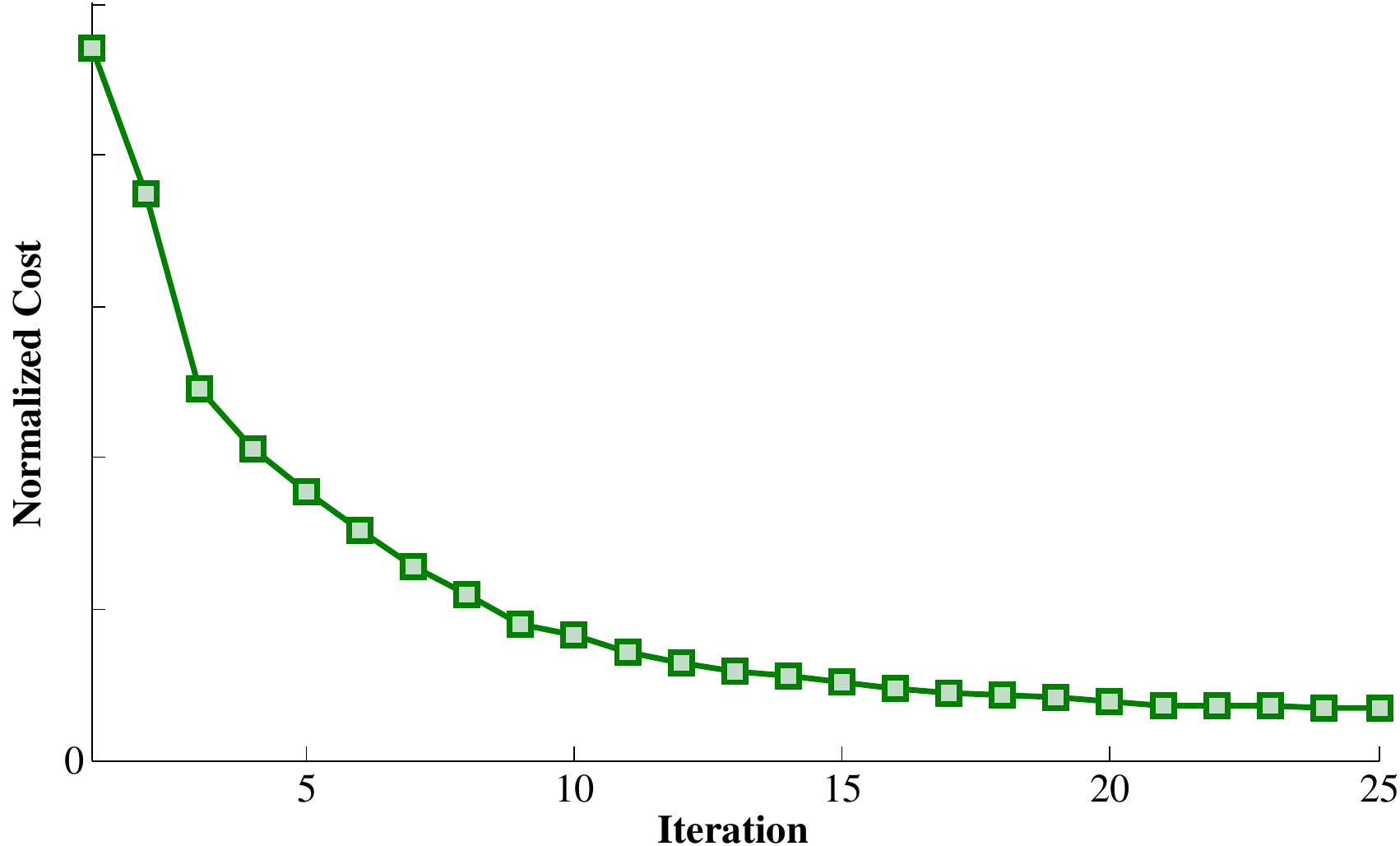}
  \caption{  \small The behavior of the cost function depicted in Eq.\eqref{eqn:DR_cost_fun} for a few
  iterations of the conjugate gradient method on Grassmannian.}
  \label{fig:cost_function_example}
\end{figure}

{
Fig.~\ref{fig:affinity_example} shows the affinity matrix before and after training for eight representative classes of YTC dataset.
Bright and dark regions represent high and low similarities, respectively. In the ideal case, the affinity 
matrix should compose of eight $3 \times 3$ sub-blocks sitting on the diagonal of the matrix. 
Fig.~\ref{fig:affinity_example} suggests that after training, each class has become more compact and hence 
more discriminative. On top of this, computing distances  in low-dimensional spaces is much faster than high-dimensional ones. For example in the case of YTC, computing 10,000 distances in the high-dimensional space took 100 seconds while after dimensionality reduction 
this time is reduced to 25 seconds.
}

\section{Experimental Evaluation}
\label{sec:experiments}

We now evaluate the algorithms introduced in the previous sections on diverse standard image-set matching problems. In particular, for our kernel-based approach, we make use of the kernel Fisher Discriminant Analysis (kFDA) algorithm. kFDA is a kernel-based approach to learning a discriminative latent space. Classification in the resulting latent space is then performed with a Nearest Neighbor (NN) classifier based on the Euclidean distance. In all our experiments, the dimensionality of the latent space, for both kFDA and our dimensionality reduction scheme, was determined by cross-validation. 
{
The kernel bandwidth, \ie, $\sigma$ in Eq.\eqref{eqn:Jeffrey_Gaussian_kernel}, Eq.\eqref{eqn:Hellinger_Gaussian_kernel} and 
Eq.\eqref{eqn:Hellinger_Laplace_kernel} was chosen from $\{0.001,0.005,0.01,0.05,0.1,0.5,1\}$.}
In the remainder of this section, we refer to our different algorithms as
\begin{itemize}[leftmargin=*]
	\renewcommand{\labelitemi}{\scriptsize$$}
	\item \textbf{NN-H:} NN classifier based on the Hellinger distance.
	\item \textbf{NN-J:} NN classifier based on the J-divergence.
	\item \textbf{kFDA-HG:} kFDA with the Hellinger distance (Eq.~\eqref{eqn:Hellinger_Gaussian_kernel}).
	\item \textbf{kFDA-HL:} kFDA with the Hellinger distance (Eq.~\eqref{eqn:Hellinger_Laplace_kernel}).
	\item \textbf{kFDA-J:} kFDA with the J-divergence (Eq.~\eqref{eqn:Jeffrey_Gaussian_kernel}).
	\item \textbf{NN-H-DR:} NN classifier based on the Hellinger distance after dimensionality reduction.
	\item \textbf{NN-J-DR:} NN classifier based on the J-divergence after dimensionality reduction.	
\end{itemize}
Since, as mentioned before, our approach was motivated by techniques that exploit geometrical structures, such as SPD or Grassmann manifolds, which have proven effective for image-set matching, we compare our results against two such techniques. In particular, we make use of Grassmann Discriminant Analysis ({\bf GDA})~\cite{Hamm_ICML_2008} and of Covariance Discriminative Learning ({\bf CDL})~\cite{Wang_CVPR_2012} as baseline algorithms, both which, as us, employ kFDA to match image-sets. For CDL, the kernel function $k_\mathcal{S}:\SPD{n} \times \SPD{n} \to \mathbb{R}$ is given by $k_\mathcal{S}(\Mat{A},\Mat{B}) = \tr(\log(\Mat{A})^T\log(\Mat{B}))$, where $\log$ is the principal matrix logarithm. For GDA, the kernel function $k_\mathcal{G}:\GRASS{p}{n} \times \GRASS{p}{n} \to \mathbb{R}$ is the projection kernel defined as $k_\mathcal{G}(\Mat{A},\Mat{B}) = \|\Mat{A}^T\Mat{B}\|_F^2$. In all our experiments, we used the same image features for our approach and for GDA and CDL. Finally, for each dataset, we also report the best result found in the literature, and refer to this result as {\bf State-of-the-art}.

\begin{figure}[!tb]
  \centering
  \includegraphics[width= 0.35 \columnwidth,keepaspectratio]{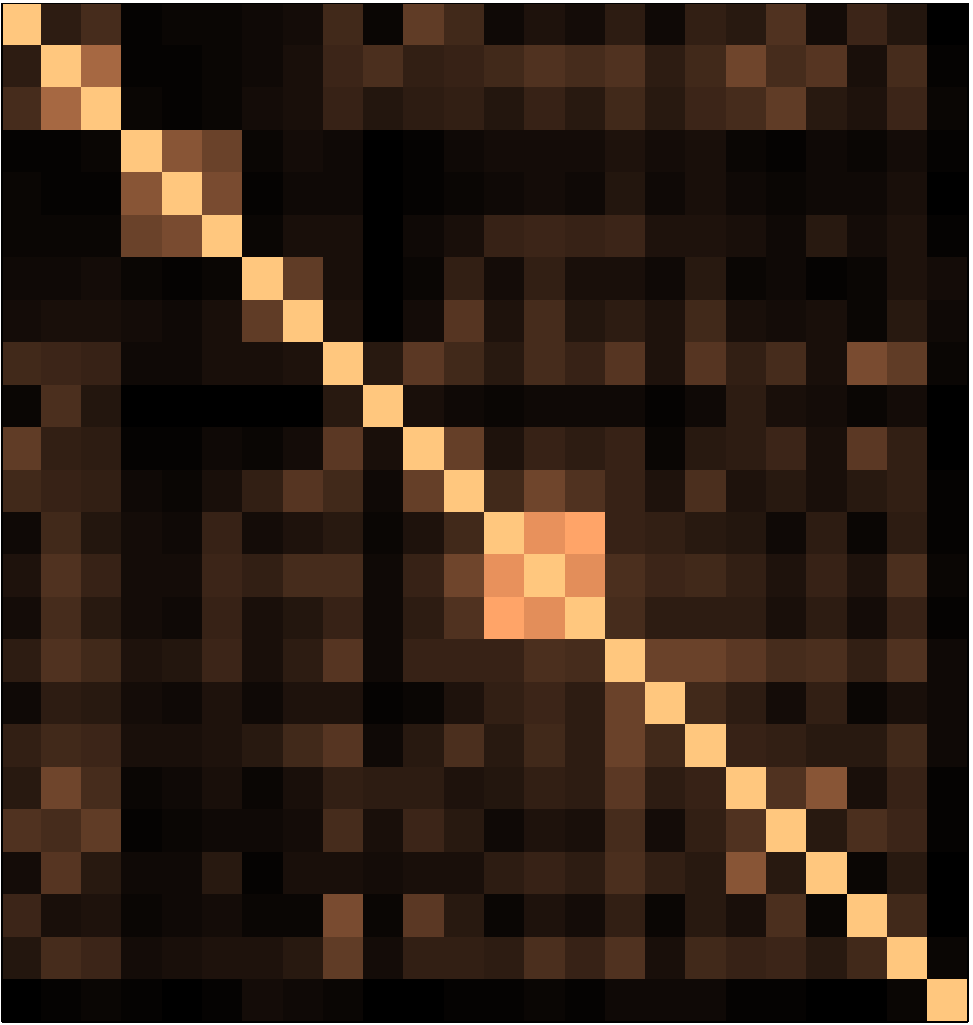}
  \hfill
  \includegraphics[width= 0.35 \columnwidth,keepaspectratio]{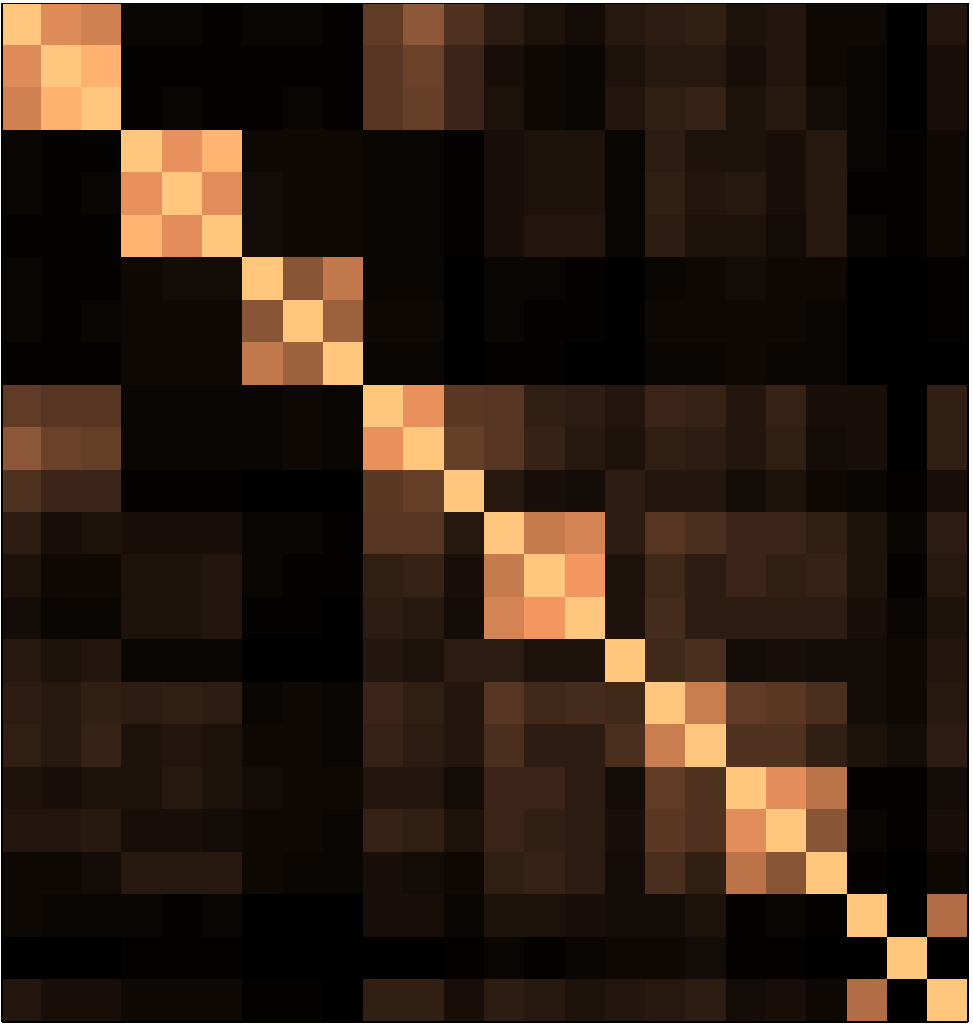}
  \caption{  \small  The affinity matrix before and after training for eight representative classes of YTC dataset.}
  \label{fig:affinity_example}
\end{figure}

\subsection{Video-Based Face Recognition}

\def \YTSIZE {0.225}
\begin{figure}[!tb] 
  \centering 
  \includegraphics[width=\YTSIZE \columnwidth,keepaspectratio]{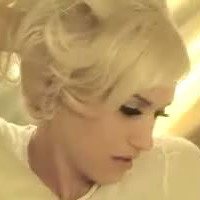}
  \hfill
  \includegraphics[width=\YTSIZE \columnwidth,keepaspectratio]{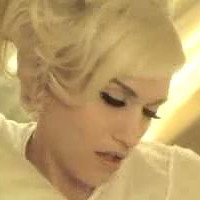}
  \hfill
  \includegraphics[width=\YTSIZE \columnwidth,keepaspectratio]{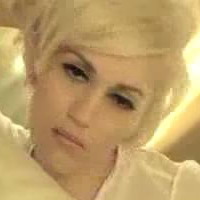}
  \hfill
  \includegraphics[width=\YTSIZE \columnwidth,keepaspectratio]{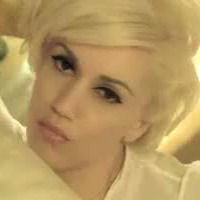}
  \caption
    {
    Samples from YouTube celebrity.
    }
  \label{fig:YT_Celebrity_example}
  \vspace{-0.15cm}
\end{figure}

For the task of image-set-based face recognition, we used the YTC~\cite{YT_Celebrity}
and COX~\cite{COX_Database} datasets (see Fig.~\ref{fig:YT_Celebrity_example} for samples from YTC). 
The YTC dataset contains 1910 video clips of 47 subjects. The COX dataset includes 1,000 subjects, each captured by three cameras (\ie, 3,000 videos in total). For the YTC dataset, we used face regions extracted from the videos, resized to $64 \times 64$, and described each region with a histogram of Local Binary Patterns (LBP)~\cite{LBP_PAMI_2002}. For the COX dataset, following~\cite{Huang_PR_2015}, we used histograms of equalized intensity values as features\footnote{We found these features to be superior to LBP features on COX.}.

For the YTC dataset, following the standard practice~\cite{Lu_ICCV_2013}, 3 videos from each person
were randomly chosen as training/gallery data, and the query set contained 6 randomly chosen videos from each subject. 
The process of random selection was repeated 5 times.
For the COX dataset, we followed the test protocol of~\cite{Huang_PR_2015}, and used videos of 100 randomly chosen subjects as training data, used only for dimensionality reduction. 
Then, 100 subjects were randomly chosen to form the gallery/probe sets for 6 different experiments. For each experiment, the camera number determines the gallery and probe sets. For example, COX12 refers to the test scenario where videos captured by Cam1 and Cam2 are used as gallery and probe set, respectively. The random selection of training and gallery/probe sets was repeated 10 times.

Table~\ref{tab:crr_face} shows the average accuracies of all methods on the YTC and COX datasets. For these datasets, the State-of-the-art baselines correspond to the metric learning approach of~\cite{Lu_ICCV_2013} and the hybrid solution of~\cite{Huang_PR_2015}, respectively. As far as geometrical methods are concerned, the results evidence that making use of the statistical manifold yields superior results compared to the Grassmann and SPD manifolds. This is even true for the direct NN classifiers based on our divergences, which are further improved by dimensionality reduction.
This, we believe, demonstrates the benefits of relying on more accurate PDF representations of each image-set (\ie, KDE in our case versus single Gaussians for CDL and GDA). Furthermore, on both datasets, our algorithms either match or outperform the state-of-the-art. 
{The exceptions are COX12 and COX32}, which could be attributed to the more sophisticated classification scheme used in~\cite{Huang_PR_2015}. Note that, as acknowledged in~\cite{Huang_PR_2015}, the hybrid method does not scale up well to large datasets. 

\setlength{\tabcolsep}{0.25em} 
\begin{table*}[!tb]	
	\scriptsize
	\centering
	\caption{\small	Average recognition rates on the YTC and COX datasets.    }
   \label{tab:crr_face}
  \begin{tabular}{|l| c|c|c|c|c|c|c|}
    \hline
    {\bf Method}   			&{\bf YTC} 	&{\bf COX12}		&{\bf COX13} 	&{\bf COX23}	&{\bf COX21}	&{\bf COX31}	&{\bf COX32}\\
    \hline
    \textbf{GDA~\cite{Hamm_ICML_2008}}			& $66.2 \pm 9.7$		& $68.8$		& $77.7$		& $71.6$
    																	& $66.0$		& $76.1$		& $74.8$\\				
    \textbf{CDL~\cite{Wang_CVPR_2012}}			& $70.9 \pm 3.2$		& $78.4$		& $85.3$		& $79.7$
    																	& $75.6$		& $85.8$		& $81.9$\\				
    \textbf{State-of-the-art}	
    &$78.2$    &$\bf 95.1$	&$96.3$    &$94.2$	&$92.3$	    &$95.4$	&$\bf 94.5$ \\     			
    \hline
    \hline
    \textbf{NN-H}        & $77.3 \pm 4.5$  	       & $61.7 \pm 4.2$		& $69.2 \pm 4.0$		& $63.5 \pm 2.3$		& $66.6 \pm 5.0$ 	 
    & $64.2 \pm 4.2$	 & $64.0 \pm 3.5$	\\ 
	\textbf{NN-J}        & $76.7 \pm 5.1$  	       & $64.7 \pm 4.1$		& $69.5 \pm 3.3$		& $63.0 \pm 2.4$		
	& $65.5 \pm 5.1$     & $70.0 \pm 3.9$		   & $63.3 \pm 3.6$	\\
	\hline
	\textbf{kFDA-HG}        & $78.9 \pm 3.4$  	       & $90.8 \pm 3.0$            	& $96.0 \pm 1.7$	& $92.9 \pm 1.9$
												   & $\bf 92.5 \pm 2.4$				& $95.8 \pm 1.7$	& $93.4 \pm 1.7$		\\			 
	\textbf{kFDA-HL}        & $78.6 \pm 4.7$  	       & $92.4 \pm 2.1$            	& $\bf 96.8 \pm 0.7$	& $\bf 94.7 \pm 1.1$
												   & $92.2 \pm 1.1$				& $\bf 96.6 \pm 0.8$	& $93.7 \pm 1.3$		\\				  	\textbf{kFDA-J}         & $\bf 79.4 \pm 3.8$  	   	   & $91.5 \pm 3.0$   		    & $95.9 \pm 1.7$	& $93.0 \pm 2.0$
					   						   & $92.5 \pm 2.5$				& $95.6 \pm 1.5$	& $93.5 \pm 1.6$		\\				     
	\hline
	\textbf{NN-H-DR}      & $78.3 \pm 3.7$  	       & $71.1 \pm 4.0$		& $83.6 \pm 3.5$		& $77.1 \pm 4.1$		
	& $76.6 \pm 3.6$	& $76.4 \pm 4.5$		& $77.1 \pm 3.5$\\
	\textbf{NN-J-DR}      & $ 79.3 \pm 3.6$  	       & $72.3 \pm 2.9$		& $82.6 \pm 3.3$		& $75.6 \pm 3.1$		
	& $73.2 \pm 3.1$	& $81.8 \pm 2.7$		& $70.4 \pm 3.8$\\
    \hline
  \end{tabular}
\end{table*}


\subsection{Dynamic Texture Recognition}

As a second task, we consider the problem of dynamic texture recognition using the DynTex++ dataset~\cite{DynTex_Dataset}.
Dynamic textures are videos of moving scenes that exhibit certain stationary properties in the time domain~\cite{DynTex_Dataset}.
Such videos are pervasive in various environments, such as sequences of rivers, clouds, fire, swarms of birds and crowds.
DynTex++~\cite{DynTex_Dataset} is comprised of 36 classes,
each of which contains 100 sequences .
We split the dataset into training and testing sets by randomly assigning half of the videos of each class to 
the training set and using the rest as query data. 
We used the LBP-TOP~\cite{LBPTOP:PAMI:2007} approach to represent each video sequence.
Table~\ref{tab:crr_texture} shows the average accuracies for 10 random splits.
To the best of our knowledge,~\cite{Rivera_TPAMI_2015} reported the highest accuracy on this dataset. 
Our kFDA-J approach yields a 1.4\% improvement over this state-of-the-art. As before, we can observe a large gap between our approach and GDA or CDL.

\begin{figure}[!tb]
  \begin{minipage}{1 \columnwidth}
    \centering
    \includegraphics[width=0.32 \columnwidth]{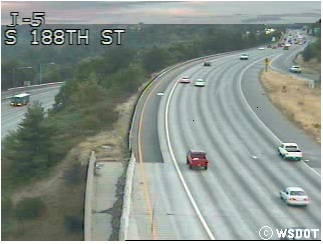}
    \hfill
    \includegraphics[width=0.32 \columnwidth]{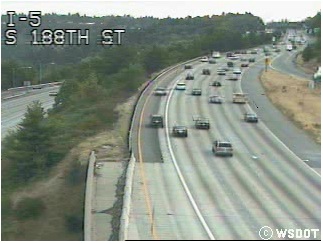}
    \hfill
    \includegraphics[width=0.32 \columnwidth]{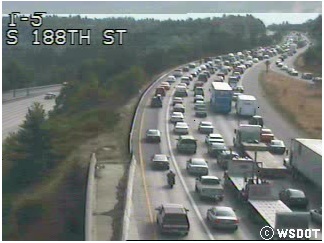}
    \vspace{-0.2cm}
    \caption
      {
      \small
      Representative examples of the three classes (light, medium, and heavy) in the UCSD traffic video dataset~\cite{Chan_CVPR2005}.
      }
    \label{fig:Traffic_example}
  \end{minipage}
\end{figure}

\setlength{\tabcolsep}{0.25em} 
\begin{table*}[!tb]	
	\footnotesize
	\centering
	\caption{\small	Average recognition rates on the DynTex++, UCSD traffic and Maryland (abbreviated as ML) scene datasets.    }
   \label{tab:crr_texture}
  \begin{tabular}{|l| c|c|c|c|}
    \hline
    {\bf Method}   			
    &{\bf DynTex++}	 &{\bf Traffic} 	&{\bf ML-LOO}		&{\bf ML}\\
    \hline
    \textbf{GDA~\cite{Hamm_ICML_2008}}			& $89.9 \pm 0.6$		& $92.5 \pm 2.6$		& $81.5$		& $70.3 \pm 5.2$ \\				
    \textbf{CDL~\cite{Wang_CVPR_2012}}			& $89.0 \pm 0.9$		& $91.7 \pm 1.9$		& $86.5$		& $76.7 \pm 7.8$	\\				
    \textbf{State-of-the-art}				    &$93.8$	    			& $95.6$		    	& $77.7$		& $NA$	\\     			
    \hline
    \hline
    \textbf{NN-H}        	 & $91.6 \pm 0.7$		   & $91.3 \pm 4.2$	& $76.9$				& $71.2 \pm 3.1$	\\ 
	\textbf{NN-J}         	 & $91.4 \pm 0.7$		   & $91.0 \pm 4.5$	& $77.7$			& $71.4 \pm 3.0$	\\
	\hline
	\textbf{kFDA-HG}        & $94.7 \pm 0.4$		& $96.1 \pm 1.5$  & $85.4$            	& $78.1 \pm 4.4$			\\			 
	\textbf{kFDA-HL}        & $94.9 \pm 0.7$		& $96.5 \pm 1.5$  & $\bf 87.7$          		& $79.0 \pm 3.1$			\\				  	\textbf{kFDA-J}          & $\bf 95.2 \pm 0.6$		& $\bf 97.3 \pm 1.4$  & $86.9$   		    & $77.8 \pm 5.3$			\\				     
	\hline
	\textbf{NN-H-DR}      	& $92.3 \pm 0.5$		& $94.9 \pm 2.9$		& $80.8$			& $79.7 \pm 4.5$	\\
	\textbf{NN-J-DR}      	 & $91.9 \pm 0.5$		& $95.6 \pm 1.5$		& $82.3$			& $\bf 80.2 \pm 3.7$	\\
    \hline
  \end{tabular}
\end{table*}

\subsection{Scene Classification}

For scene classification, we employed the UCSD traffic dataset~\cite{Chan_CVPR2005} and 
the Maryland~\cite{Shroff_CVPR_2010} scene recognition dataset.
The UCSD traffic dataset contains 254 video sequences of highway traffic of varying patterns (\ie, light, medium and heavy)
in various weather conditions (\eg, cloudy, raining, sunny). See Fig.~\ref{fig:Traffic_example} for examples.
We used HoG features~\cite{Dalal_CVPR_2005} to describe each frame. Our experiments
were performed using the splits provided with the dataset~\cite{Chan_CVPR2005}. We report the average accuracies over these splits in 
Table~\ref{tab:crr_texture}. The state-of-the-art results were reported in~\cite{Ravichandran_ACCV_2011}.
Once again, we see that our kernel-based and dimensionality reduction algorithms comfortably outperform GDA and CDL, as well as the previous state-of-the-art.

As a last experiment, we used the Maryland dataset, which contains 13 different classes of dynamic scenes, such as avalanches 
and tornados. See Fig.~\ref{fig:Maryland_example} for examples.
This dataset is more challenging, and we observed that handcrafted features, such as LBP or HoG, do not provide sufficiently discriminative representations. Therefore, we used the last layer of the CNN trained in~\cite{Xhou_NIPS_2014} as frame descriptors.
We then reduced the dimensionality of the CNN output  to 400 using PCA. We first employed the standard Leave-One-Out (LOO) test protocol. Furthermore, we also evaluated the methods on 10 different training/query partitions obtained by randomly choosing 70\% of the dataset for training and the remaining 30\% for testing. The average classification accuracies for both protocols are reported in Table~\ref{tab:crr_texture}. Note that no state-of-the-art results has been reported in the literature on our second test protocol. Our approach outperforms the state-of-the-art result of Feichtenhofer~\cite{Feichtenhofer_CVPR_2014} by more than 7\%. 
While this may be attributed in part to the CNN features, note that our approach still outperforms GDA and CDL based on the same features.

\begin{figure}[!tb]
  \begin{minipage}{1 \columnwidth}
    \centering
    \includegraphics[width=0.32 \columnwidth]{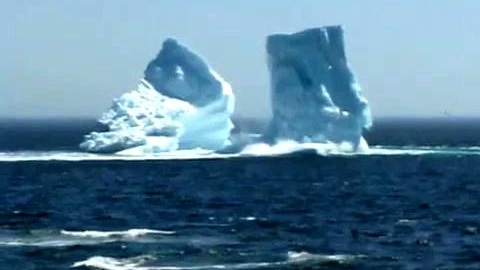}
    \hfill
    \includegraphics[width=0.32 \columnwidth]{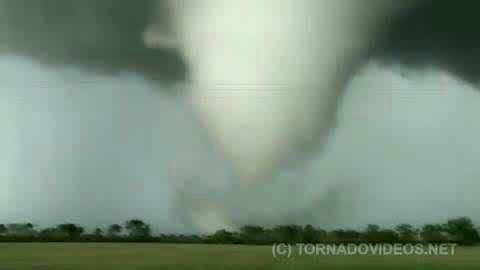}
    \hfill
    \includegraphics[width=0.32 \columnwidth]{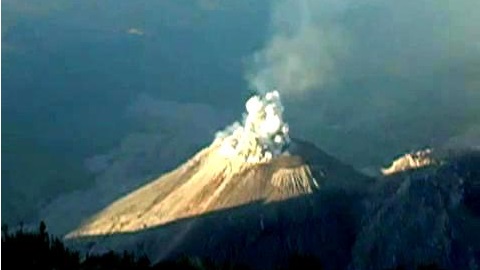}
    \caption
      {
      \small
      Representative examples of three classes of the Maryland scene dataset~\cite{Shroff_CVPR_2010}.
      From left to right: iceberg collapsing, tornado, and volcano eruption.
      }
    \label{fig:Maryland_example}
  \end{minipage}
\end{figure}

\section{Conclusions and Future Work}
\label{sec:conclusion}

We have introduced a novel framework to model and compare image-sets. Specifically, we have made use of KDE to represent an image-set with its PDF, and have proposed practical solutions to employ $f$-divergences for image-set matching, including empirical estimates of $f$-divergences, valid \emph{pd} kernels on the statistical manifold and a supervised dimensionality reduction algorithm inherently accounting for $f$-divergences in the resulting latent space.
In the future, we plan to extend our learning scheme to the unsupervised and semi-supervised scenarios. Furthermore, we intend to study the use and effectiveness of other divergences to tackle the problem of image-set matching.

\appendix
\section{Gradient for Dimensionality Reduction}

The cost function to be minimized by our dimensionality reduction approach, given in Eq.~15 and Eq.~16 of the main paper, can be expressed as
\begin{eqnarray}
\Mat{W}^* = &\underset{\Mat{W}}{\operatorname{argmin}}&\hspace{-0.2cm} 
\sum_{i,j} a({\Mat{X}_i,\Mat{X}_j}) \cdot \delta\hspace{-0.05cm}\left(\Mat{W}^T\Mat{X}_i,\Mat{W}^T\Mat{X}_j\right)
\label{eqn:DR_opt_prob} \\
& {\rm s.t.} &\hspace{-0.5cm}\Mat{W}^T \Mat{W} = \mathbf{I}_d\;,\notag
\end{eqnarray}
where $a(\cdot,\cdot)$ is an affinity measure that does not depend on $\Mat{W}$, $\delta: \mathcal{M} \times \mathcal{M} \to \mathbb{R}^+$ is either the Hellinger distance $\delta_H^2(\cdot,\cdot)$, or the Jeffrey divergence $\delta_{J}(\cdot,\cdot)$, and the sum is over all pairs of training image-sets. 

As shown in Eq.~17 of the main paper, to minimize this cost function by a gradient descent method on the Grassmannian, we need to compute the partial derivatives
${\partial \delta_H^2(\cdot,\cdot)}/{\partial \Mat{W}}$ or ${\partial \delta_J(\cdot,\cdot)}/{\partial \Mat{W}}$. Since both ${\delta_H^2(\cdot,\cdot)}$ and 
${\delta_J(\cdot,\cdot)}$ are expressed in terms of $T(\cdot)$ (defined in Eq.~9 of the main paper), we start by deriving $\partial T /\partial \Mat{W}$.
\begin{align}
\frac{\partial}{\partial \Mat{W}} T(\Mat{W}^T \Vec{x}) &= \frac{\partial}{\partial \Mat{W}}
\frac{p(\Mat{W}^T \Vec{x})}{p(\Mat{W}^T \Vec{x}) + q(\Mat{W}^T \Vec{x})} \notag \\
&= \frac{1}{\Big(p(\Mat{W}^T \Vec{x}) + q(\Mat{W}^T \Vec{x})\Big)^2} \bigg(
q(\Mat{W}^T \Vec{x}) \frac{\partial p(\Mat{W}^T \Vec{x})}{\partial \Mat{W}} - 
p(\Mat{W}^T \Vec{x}) \frac{\partial q(\Mat{W}^T \Vec{x})}{\partial \Mat{W}}
\bigg)\;. \label{eqn:grad_tx} 
\end{align}
From the definition of the empirical estimate of a distribution given in Eq.~4 of the main paper, the derivative of $\partial p(\Mat{W}^T \Vec{x})/ \partial \Mat{W} $ (and similarly for $\partial q(\Mat{W}^T \Vec{x})/ \partial \Mat{W} $) can be written as
\begin{align}
&\frac{\partial}{\partial \Mat{W}} p(\Mat{W}^T \Vec{x}) = \frac{\partial}{\partial \Mat{W}} \Bigg\{
\frac{1}{n_p\sqrt{\det(2\pi \Sigma_p)}}\sum_{i = 1}^{n_p} 
\exp \bigg(-\frac{1}{2} \Big(\Vec{x} - \Vec{x}_i^{(p)} \Big)^T \Mat{W}\Sigma_p^{-1} \Mat{W}^T\Big(\Vec{x} - \Vec{x}_i^{(p)} \Big)\bigg)
\Bigg\} \notag \\
&= \frac{-1}{n_p\sqrt{\det(2\pi \Sigma_p)}}\sum_{i = 1}^{n_p} \Bigg\{\exp \bigg(-\frac{1}{2} \Big(\Vec{x} - \Vec{x}_i^{(p)} \Big)^T \Mat{W}\Sigma_p^{-1} \Mat{W}^T\Big(\Vec{x} - \Vec{x}_i^{(p)} \Big)\bigg) \Big(\Vec{x} - \Vec{x}_i^{(p)} \Big)\Big(\Vec{x} - \Vec{x}_i^{(p)} \Big)^T
\Bigg\}\Mat{W}\Sigma_p^{-1} \label{eqn:grad_p}. 
\end{align}

According to the definition of the Hellinger distance given by Eq.~8 and Eq.~10 of the main paper, this lets us derive ${\partial \delta_H^2(\cdot,\cdot)}/{\partial \Mat{W}}$ as
\begin{small}
\begin{align}
	\frac{\partial}{\partial \Mat{W}}\delta_H^2(p \Vert q) &=  
	\frac{\partial}{\partial \Mat{W}}E_{p} \Bigg\{\bigg(\sqrt{T\big(\Mat{W}^T\Vec{x}\big)} - 
	\sqrt{1 - T\big(\Mat{W}^T\Vec{x}\big)}\bigg)^2 \Bigg\} \notag \\ &
	+\frac{\partial}{\partial \Mat{W}}E_{q} \Bigg\{\bigg(\sqrt{T\big(\Mat{W}^T\Vec{x}\big)} - 
	\sqrt{1 - T\big(\Mat{W}^T\Vec{x}\big)}\bigg)^2 \Bigg\}\notag\\
	&= E_{p} \Bigg\{ 2\bigg(\sqrt{T\big(\Mat{W}^T\Vec{x}\big)} - 
	\sqrt{1 - T\big(\Mat{W}^T\Vec{x}\big)}\bigg)
	\Bigg( \frac{1}{2\sqrt{T\big(\Mat{W}^T\Vec{x}\big)}} 
	+ \frac{1}{2\sqrt{1 - T\big(\Mat{W}^T\Vec{x}\big)}}\Bigg)
	\frac{\partial}{\partial \Mat{W}} T\big(\Mat{W}^T \Vec{x}\big)
	\Bigg\} \notag\\
	&+ E_{q} \Bigg\{ 2\bigg(\sqrt{T\big(\Mat{W}^T\Vec{x}\big)} - \sqrt{1 - T\big(\Mat{W}^T\Vec{x}\big)}\bigg)
	\Bigg( \frac{1}{2\sqrt{T(\Mat{W}^T\Vec{x}\big)}} + \frac{1}{2\sqrt{1 - T(\Mat{W}^T\Vec{x}\big)}}\Bigg)
	\frac{\partial}{\partial \Mat{W}} T(\Mat{W}^T \Vec{x}\big)
	\Bigg\} \notag \\
	&= E_{p} \Bigg\{ \sqrt{\frac{2T\big(\Mat{W}^T\Vec{x}\big)-1}{T\big(\Mat{W}^T\Vec{x}\big)\Big(1-T\big(\Mat{W}^T\Vec{x}\big)\Big)}} 
	\frac{\partial}{\partial \Mat{W}} T\big(\Mat{W}^T \Vec{x}\big)\Bigg\} \notag \\ &+
	E_{q} \Bigg\{ \sqrt{\frac{2T\big(\Mat{W}^T\Vec{x}\big)-1}{T\big(\Mat{W}^T\Vec{x}\big)\Big(1-T\big(\Mat{W}^T\Vec{x}\big)\Big)}} 
	\frac{\partial}{\partial \Mat{W}} T\big(\Mat{W}^T \Vec{x}\big)\Bigg\} \notag \\
	&= \frac{1}{n_p}\sum_i^{n_p} \Bigg\{ \sqrt{\frac{2T\big(\Mat{W}^T\Vec{x}_i^{(p)}\big)-1}
	{T\big(\Mat{W}^T\Vec{x}_i^{(p)}\big)\Big(1-T\big(\Mat{W}^T\Vec{x}_i^{(p)}\big)\Big)}} 
	\frac{\partial}{\partial \Mat{W}} T(\Mat{W}^T \Vec{x}_i^{(p)}\big) \Bigg\} \notag \\ &+
	\frac{1}{n_q}\sum_i^{n_q} \Bigg\{ \sqrt{\frac{2T\big(\Mat{W}^T\Vec{x}_i^{(q)}\big)-1}
	{T\big(\Mat{W}^T\Vec{x}_i^{(q)}\big)\Big(1-T\big(\Mat{W}^T\Vec{x}_i^{(q)}\big)\Big)}} 
	\frac{\partial}{\partial \Mat{W}} T\big(\Mat{W}^T \Vec{x}_i^{(q)}\big) \Bigg\}\;,
	\label{eqn:gradient_hellinger}
\end{align}
\end{small}
which depends on $\partial T /\partial \Mat{W}$ obtained from Eq.~\ref{eqn:grad_tx} and Eq.~\ref{eqn:grad_p} above.

Similarly, for the Jeffrey divergence defined in Eq.~11 of the main paper, we can obtain ${\partial \delta_J(\cdot,\cdot)}/{\partial \Mat{W}}$ as
\begin{small}
\begin{align}
	\frac{\partial}{\partial \Mat{W}}\delta_J(p \Vert q) &=  	
	\frac{\partial}{\partial \Mat{W}}E_{p} \Bigg\{ 
	\Big(2T\big(\Mat{W}^T\Vec{x}\big)-1\Big)
	\log \Bigg(\frac{T\big(\Mat{W}^T\Vec{x}\big)}{1-T\big(\Mat{W}^T\Vec{x}\big)}\Bigg) \Bigg\} \notag \\ &+ 
	\frac{\partial}{\partial \Mat{W}}E_{q} \Bigg\{ 
	\Big(2T\big(\Mat{W}^T\Vec{x}\big)-1\Big)
	\log \Bigg(\frac{T\big(\Mat{W}^T\Vec{x}\big)}{1-T\big(\Mat{W}^T\Vec{x}\big)}\Bigg) \Bigg\} \notag \\
	&= E_{p} \Bigg\{\Bigg(2 \log \Bigg(\frac{T\big(\Mat{W}^T\Vec{x}\big)}{1-T\big(\Mat{W}^T\Vec{x}\big)}\Bigg)
	+ 	\Big(2T\big(\Mat{W}^T\Vec{x}\big)-1\Big)
	\bigg(\frac{1}{T\big(\Mat{W}^T\Vec{x}\big)} + \frac{1}{1 - T\big(\Mat{W}^T\Vec{x}\big)}\bigg)\Bigg)
	\frac{\partial}{\partial \Mat{W}} T\big(\Mat{W}^T \Vec{x}\big)
	\Bigg\} \notag \\ &+
	E_{q} \Bigg\{\Bigg(2 \log \Bigg(\frac{T\big(\Mat{W}^T\Vec{x}\big)}{1-T\big(\Mat{W}^T\Vec{x}\big)}\Bigg)
	+ 	\Big(2T\big(\Mat{W}^T\Vec{x}\big)-1\Big)
	\bigg(\frac{1}{T\big(\Mat{W}^T\Vec{x}\big)} + \frac{1}{1 - T\big(\Mat{W}^T\Vec{x}\big)}\bigg)\Bigg)
	\frac{\partial}{\partial \Mat{W}} T\big(\Mat{W}^T \Vec{x}\big)
	\Bigg\} \notag \\	
	&= \frac{1}{n_p}\sum_i^{n_p} \Bigg\{\Bigg(2 \log \Bigg(\frac{T\big(\Mat{W}^T\Vec{x}_i^{(p)}\big)}
	{1-T\big(\Mat{W}^T\Vec{x}_i^{(p)}\big)}\Bigg)
	+ 	\frac{2T\big(\Mat{W}^T\Vec{x}_i^{(p)}\big)-1}{T\big(\Mat{W}^T\Vec{x}_i^{(p)}\big)\Big(1 - T\big(\Mat{W}^T\Vec{x}_i^{(p)}\big)\Big)}	
	\Bigg)\frac{\partial}{\partial \Mat{W}} T\big(\Mat{W}^T \Vec{x}_i^{(p)}\big)
	\Bigg\} \notag \\ &+
	\frac{1}{n_q}\sum_i^{n_q} \Bigg\{\Bigg(2 \log \Bigg(\frac{T\big(\Mat{W}^T\Vec{x}_i^{(q)}\big)}
	{1-T\big(\Mat{W}^T\Vec{x}_i^{(q)}\big)}\Bigg)
	+ 	\frac{2T\big(\Mat{W}^T\Vec{x}_i^{(q)}\big)-1}{T\big(\Mat{W}^T\Vec{x}_i^{(q)}\big)\Big(1 - T\big(\Mat{W}^T\Vec{x}_i^{(q)}\big)\Big)}	
	\Bigg)\frac{\partial}{\partial \Mat{W}} T\big(\Mat{W}^T \Vec{x}_i^{(q)}\big)\;.
	\Bigg\} 
	\label{eqn:gradient_jeffrey}
\end{align}
\end{small}
Using either of these derivatives lets us compute the gradient of the objective function on the Grassmann manifold from Eq.~17 of the main paper, and thus ultimately learn the mapping $\Mat{W}$ to a lower-dimensional subspace where the $f$-divergence reflects the chosen affinity measure.

{
\bibpunct{(}{)}{;}{a}{,}{,}
\bibliographystyle{natbib}
\bibliography{references}
}

\end{document}